\documentclass[11pt]{llncs}
\pdfoutput=1
\usepackage{a4,a4wide}
\usepackage{graphicx}
\usepackage{amsmath}
\usepackage{mathrsfs}

\newcommand{\mye}[1]{\ensuremath{\mathit{#1}}}
\newcommand{\kb}{\mye{K}}
\newcommand{\true}{\mye{true}}

\newcommand{\naf}{\mye{not}}
\newcommand{\predicate}[1]{\textit{#1}}
\newcommand{\constant}[1]{\textsf{#1}}
\newcommand{\policy}{\mye{policy}}
\newcommand{\prior}{\mye{prior}}

\title{Confidentiality-Preserving Data Publishing for Credulous Users by Extended Abduction}
\author{Katsumi Inoue\inst{1}, Chiaki Sakama\inst{2} and Lena Wiese\inst{1}\thanks{Lena Wiese gratefully acknowledges a postdoctoral research grant of the German Academic Exchange Service (DAAD).}}
\institute{National Institute of Informatics\\2-1-2 Hitotsubashi, Chiyoda-ku, Tokyo 101-8430, Japan\\\texttt{\{ki|wiese\}@nii.ac.jp}
\and
Department of Computer and Communication Sciences 
    Wakayama University\\930 Sakaedani, Wakayama 640-8510, Japan\\\texttt{sakama@sys.wakayama-u.ac.jp}}

\begin{document}

\maketitle

\begin{abstract}
Publishing private data on external servers incurs the problem of how to avoid unwanted disclosure of confidential data.
We study a problem of confidentiality in extended disjunctive logic programs and show how it can be solved by extended abduction. 
In particular, we analyze how credulous non-monotonic reasoning affects confidentiality.
\end{abstract}%

\section{Introduction}

Confidentiality of data (also called privacy or secrecy in some contexts) is a major security goal.
Releasing data to a querying user without disclosing confidential information has long been investigated in areas like 
access control, $k$-anonymity, inference control, and data fragmentation. 
Such approaches prevent disclosure according to some security policy by restricting data access (denial, refusal), by modifying some data (perturbation, noise addition, cover stories, lying, weakening), or by breaking sensitive associations (fragmentation).
Several approaches (like \cite{bonatti1995foundations,cuencagrau2008privacy-preserving,stouppa2009data,toland2010inferenceproblem,biskup2010usability,wiese2010horizontal}) 
employ logic-based mechanisms to ensure data confidentiality.
In particular, \cite{dix2005relationship} use brave reasoning in default logic theories to solve a privacy problem in a classical database (a set of ground facts).
For a non-classical knowledge base (where negation as failure \naf{} is allowed) \cite{zhao2010usingasp} study correctness of access rights. Confidentiality of predicates in collaborative multi-agent abduction is a topic in \cite{ma2011multiagent}.

In this article we analyze \textbf{confidentiality-preserving data publishing} in a knowledge base setting: data as well as integrity constraints or deduction rules are represented as logical formulas.
If such a knowledge base is released to the public for general querying (e.g., microcensus data) or outsourced to a storage provider (e.g., database-as-a-service in  cloud computing), confidential data could be disclosed. We assume that users accessing the published knowledge base use a form of credulous (also called brave) reasoning to retrieve data from it; users also possess some invariant ``a priori knowledge'' that can be applied to these data to deduce further information. On the knowledge base side, a confidentiality policy specifies which is the confidential information that must never be disclosed.
This paper is one of only few papers (see \cite{sakama2011dishonestreasoning,zhao2010usingasp,ma2011multiagent}) covering confidentiality for logic programs.  This formalism however has relevance in multi-agent communications where agent knowledge is modeled by logic programs. 
With \textbf{extended abduction} (\cite{sakama2003abductive}) we obtain a ``secure version'' of the knowledge base that can safely be published even when a priori knowledge is applied.  We show that computing the secure version for a credulous user corresponds to finding a skeptical anti-explanation for all the elements of the confidentiality policy.
Extended abduction has been used in different applications like for example providing a logical framework for dishonest reasoning \cite{sakama2011dishonestreasoning}. It can be solved by computing the answer sets of an update program (see \cite{sakama2003abductive}); thus an implementation of extended abduction can profit from current answer set programming (ASP) solvers \cite{calimeri2011aspcompetition}.
To retrieve the confidentiality-preserving knowledge base $K^\mathit{pub}$ from the input knowledge base $K$, the a priori knowledge \prior{} and the confidentiality policy \policy, a row of transformations are applied; the overall approach is depicted in Figure~\ref{fig:confidentialitycredulous}.

In sum, this paper makes the following contributions:
\begin{itemize}
  \item it formalizes confidentiality-preserving data publishing for a user who retrieves data under a credulous query response semantics.
	\item it devises a procedure to securely publish a logic program (with an expressiveness up to extended disjunctive logic programs) respecting a subset-minimal change semantics.
	\item it shows that confidentiality-preservation for credulous users corresponds to finding a skeptical anti-explanation and can be solved by extended abduction.
\end{itemize}

In the remainder of this article, Section~\ref{sec:edps} provides background on extended disjunctive logic programs and answer set semantics; Section~\ref{sec:confpreservingkbs} defines the problem of confidentiality  in data publishing; Section~\ref{sec:extendedabduction} recalls extended abduction and update programs; Section~\ref{sec:confidentialityups} shows how answer sets of update programs correspond to confidentiality-preserving knowledge bases; and Section~\ref{sec:conclusion} gives some discussion and concluding remarks.

\section{EDPs and answer set semantics}\label{sec:edps}
In this article, a knowledge base \kb{} is represented by an \emph{extended disjunctive logic program} (EDP) --  a set of formulas called \emph{rules} of the form:
\begin{displaymath}
 L_1;\mathellipsis;L_l\leftarrow L_{l+1},\mathellipsis,L_m,\naf L_{m+1},\mathellipsis,\naf L_n   ~~~~~(n\geq m\geq l\geq 0)
\end{displaymath}
A rule contains literals $L_i$, disjunction ``;'', conjunction ``,'', negation as failure ``\naf'', and material implication ``$\leftarrow$''.
A literal is a first-order atom or an atom preceded by classical negation ``$\neg$''. $\naf L$ is called a \emph{NAF-literal}.
The disjunction left of the implication $\leftarrow$ is called the \emph{head}, while the conjunction right of $\leftarrow$ is called the \emph{body} of the rule. For a rule $R$, we write $\mathit{head}(R)$ to denote the set of literals $\{ L_1,\mathellipsis,L_l\}$ and $\mathit{body}(R)$ to denote the set of (NAF-)literals $\{L_{l+1},\mathellipsis,L_m,\naf L_{m+1},\mathellipsis,\naf L_n\}$.
 Rules consisting only of a singleton head $L\leftarrow$ are identified with the literal $L$ and used interchangeably.
An EDP is ground if it contains no variables. If an EDP contains variables, it is identified with the set of its ground instantiations: the elements of its Herbrand universe are substituted in for the variables in all possible ways. We assume that the language contains no 
function symbol, so that each rule with variables represents a finite set of ground rules.
For a program $\kb$, we denote $\mathscr{L}_\kb$ the set of ground literals in the language of $\kb$.
Note that EDPs offer a high expressiveness including disjunctive and non-monotonic reasoning.

\begin{figure}[t]
\begin{center}
\includegraphics{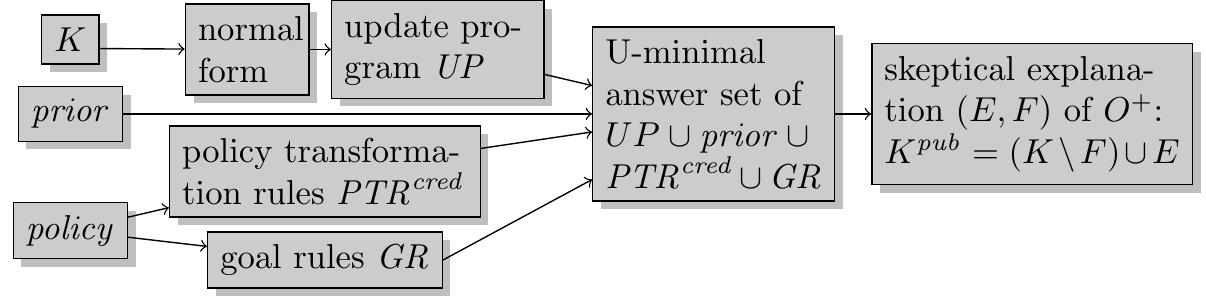}
\caption{Finding a confidentiality-preserving $K^\mathit{pub}$ for a credulous user\label{fig:confidentialitycredulous}}
\end{center}
\end{figure} 

\begin{example}
In a medical knowledge base  $\predicate{Ill}(x,y)$ states that a patient $x$ is ill with disease $y$;
 $\predicate{Treat}(x,y)$ states that $x$ is treated with medicine $y$. 
 Assume that if you read the record and find that one treatment (\constant{Medi1}) is recorded and another one  (\constant{Medi2}) is not recorded, then you know that the patient is at least ill with  \constant{Aids} or \constant{Flu} (and possibly has other illnesses).\\ $\kb=\{\predicate{Ill}(x,\constant{Aids});\predicate{Ill}(x,\constant{Flu})\leftarrow\predicate{Treat}(x,\constant{Medi1}),\naf \predicate{Treat}(x,\constant{Medi2})~,$\\ $~~~~~~~~~~\predicate{Ill}(\constant{Mary},\constant{Aids})~,~\predicate{Treat}(\constant{Pete},\constant{Medi1})\}~~$ serves as a running example.
\end{example}
The semantics of $\kb$ can be given by the answer set semantics \cite{gelfond1991classical}:
A set $S\subseteq \mathscr{L}_\kb$ of ground literals \emph{satisfies} a ground literal $L$ if $L\in S$; $S$ satisfies a conjunction if it satisfies every conjunct; $S$ satisfies a disjunction if it satisfies at least one disjunct; $S$ satisfies a ground rule if whenever the body literals are contained in $S$ ($\{ L_{l+1},\mathellipsis, L_m\}\subseteq S$) and all NAF-literals are not contained in $S$ ($\{ L_{m+1},\mathellipsis, L_n\}\cap S=\emptyset$), then at least one head literal is contained in $S$ ($L_i\in S$ for an $i$ such that $1\leq i \leq l$).
If an EDP $\kb$ contains no NAF-literals ($m=n$), then such a set $S$ is an \emph{answer set} of $\kb$ if $S$ is a subset-minimal set such that 
\begin{enumerate}
 \item $S$ satisfies every rule from the ground instantiation of $\kb$,
 \item If $S$ contains a pair of complementary literals $L$ and $\neg L$, then $S=\mathscr{L}_\kb$.
\end{enumerate}
This definition of an answer set can be extended to full EDPs (containing NAF-literals) as in \cite{sakama2003abductive}: For an EDP $\kb$ and a set of ground literals $S\subseteq \mathscr{L}_\kb$,   $\kb$ can be transformed into a NAF-free program $\kb^S$ as follows. For every ground rule from the ground instantiation of $\kb$ (with respect to its Herbrand universe), the rule
$L_1;\mathellipsis ;L_l \leftarrow L_{l+1},\mathellipsis , L_m$ is in $\kb^S$ if $\{L_{m+1},\mathellipsis ,L_n\}\cap S=\emptyset$. Then, $S$ is an answer set of $\kb$ if $S$ is an answer set of $\kb^S$.
An answer set is \emph{consistent} if it is not $\mathscr{L}_\kb$. A program $\kb$ is \emph{consistent} if it has a
consistent answer set; otherwise $\kb$ is \emph{inconsistent}.
\begin{example}\label{ex:answersets}
The example \kb{} has the following two consistent answer sets 
\begin{eqnarray*}
S_1&=&\{\predicate{Ill}(\constant{Mary},\constant{Aids}),
\predicate{Treat}(\constant{Pete},\constant{Medi1}),\predicate{Ill}(\constant{Pete},\constant{Aids})\}\\
S_2&=&\{\predicate{Ill}(\constant{Mary},\constant{Aids}),
\predicate{Treat}(\constant{Pete},\constant{Medi1}),\predicate{Ill}(\constant{Pete},\constant{Flu})\}
\end{eqnarray*}
When adding the negative fact $\neg\predicate{Ill}(\constant{Pete},\constant{Flu})$ to \kb, then there is just one consistent answer set left: for
$\kb':=\kb\cup\{\neg\predicate{Ill}(\constant{Pete},\constant{Flu})\}$
 the unique answer set is
\begin{displaymath}
S'=\{\predicate{Ill}(\constant{Mary},\constant{Aids}),\neg\predicate{Ill}(\constant{Pete},\constant{Flu}),
\predicate{Treat}(\constant{Pete},\constant{Medi1}),\predicate{Ill}(\constant{Pete},\constant{Aids})\}.
\end{displaymath}
\end{example}

 If a rule $R$ is satisfied in \emph{every}
answer set of $\kb$, we write $\kb\models R$. In particular, $\kb\models L$ if a literal $L$ is
included in every answer set of $\kb$. 

\section{Confidentiality-Preserving Knowledge Bases}\label{sec:confpreservingkbs}
When publishing a knowledge base \kb{} while preserving confidentiality of some data in $\kb$ we do this according to 
\begin{itemize}
 \item the query response semantics that a user querying the published knowledge base applies; we focus on  credulous query response semantics
 \item a confidentiality policy (denoted \policy) describing confidential information that should not be released to the public
 \item  background (a priori) knowledge (denoted \prior) that a user can combine with query responses from the published knowledge base
\end{itemize}
First we define the credulous query response semantics: a ground formula $Q$ is $\true$ in $\kb$, if $Q$ is satisfied in some answer set of $\kb$ -- that is, there might be answer sets that do not satisfy $Q$. If a rule $Q$ is non-ground and contains some free variables, the credulous response of  $\kb$ is the set of ground instantiations of $Q$ that are $\true$ in $\kb$.

\begin{definition}[Credulous query response semantics]
Let $U$ be the Herbrand universe of a consistent knowledge base $\kb$. 
The \emph{credulous query responses} of formula $Q(X)$ (with a vector $X$ of free variables) in $\kb$ are
\begin{eqnarray*}
 \mathit{cred}(\kb,Q(X))&=\{Q(A)\mid &A \mbox{ is a vector of elements } a \in U
 \mbox{ and there}\\&&\mbox{is an answer set of }\kb\mbox{ that satisfies } Q(A)\}
\end{eqnarray*}
 In particular, for a ground formula $Q$, 
\begin{displaymath}
	\mathit{cred}(\kb,Q)=\left\{
	\begin{array}{cl}
	Q& \mbox{ if }\kb\mbox{ has an answer set that satisfies }Q\\
	\emptyset & \mbox{ otherwise}
\end{array}
	\right.
\end{displaymath}
\end{definition}

It is usually assumed that in addition to the query responses a user has some additional knowledge that he can apply to the query responses. 
Hence, we additionally assume given a set of rules as some \emph{invariant} \textbf{a priori knowledge} \prior{}. 
Without loss of generality we assume that \prior{} is an EDP.
Thus, the priori knowledge may consist of additional facts that the user assumes to hold in \kb, or some rules that the user can apply to data in \kb{} to deduce new information. 

A \textbf{confidentiality policy} \policy{} specifies confidential information.
We assume that \policy{} contains only conjunctions of (NAF-)literals. However, see Section~\ref{sec:policytrans} for a brief discussion on how to use  more expressive policy formulas. 
We do not only have to avoid that the published knowledge base contains confidential information but also prevent the user from deducing confidential information with the help of his a priori knowledge; this is known as the inference problem \cite{farkas2002inferenceproblem,biskup2010usability}. 
\begin{example}\label{ex:prior}
 If we wish to declare the disease aids as confidential for any patient $x$ we can do this with 
 $\policy = \{\predicate{Ill}(x,\constant{Aids})\}$.
A user querying $K^\mathit{pub}$ might know that a person suffering from flu is not able to work. Hence 
 $\prior = \{\neg\predicate{AbleToWork}(x)\leftarrow\predicate{Ill}(x,\constant{Flu})\}$.
If we wish to also declare a lack of work ability as confidential, we can add this to the confidentiality policy:
 $\policy' = \{\predicate{Ill}(x,\constant{Aids})~,~\neg\predicate{AbleToWork}(x)\}$.
\end{example}
Next, we establish a definition of confidentiality-preservation that allows for the answer set semantics as an inference mechanism and respects the credulous query response semantics:
when treating elements of the confidentiality policy as queries, the credulous responses must be empty.

\begin{definition}[Confidentiality-preservation for credulous user]\label{def:confidentialskeptical}
A knowledge base $K^\mathit{pub}$ \emph{preserves confidentiality} of a given confidentiality policy 
under the credulous query response semantics and with respect to a given a priori knowledge \prior, if for every conjunction $C(X)$ in the policy, the credulous query responses of $C(X)$ in $K^\mathit{pub}\cup\prior$ are empty:
$\mathit{cred}(K^\mathit{pub}\cup\prior,C(X))=\emptyset$.
\end{definition}

Note that in this definition the Herbrand universe of $K^\mathit{pub}\cup\prior$ is applied in the query response semantics; hence, free variables in policy elements $C(X)$ are instantiated according to this universe. Note also that $K^\mathit{pub}\cup\prior$ must be consistent. Confidentiality-preservation for \emph{skeptical} query response semantics is topic of future work.

A goal secondary to confidentiality-preservation is minimal change:
We want to publish as many data as possible and want to modify these data as little as possible.
Different notions of minimal change are used in the literature (see for example \cite{afrati2009repair} for a collection of minimal change semantics in a data integration setting). 
We apply a subset-minimal change semantics: we choose a $K^\mathit{pub}$ that differs from $K$ only subset-minimally. 
In other words, there is not other confidentiality-preserving knowledge base ${K^\mathit{pub}}'$ which inserts (or deletes) less rules to (from) $K$ than $K^\mathit{pub}$. \begin{definition}[Subset-minimal change] 
A confidentiality-preserving knowledge base $K^\mathit{pub}$ \emph{subset-minimally changes} $K$ (or is \emph{minimal}, for short) if there is no confidentiality-preserving knowledge base
 ${K^\mathit{pub}}'$ such that 
$((K\setminus {K^\mathit{pub}}')\cup({K^\mathit{pub}}'\setminus K))\subset ((K\setminus K^\mathit{pub})\cup(K^\mathit{pub}\setminus K))$.
\end{definition}

\begin{example}\label{ex:confidentialcredulous}
 For the example $\kb$ and \policy{} and no a priori knowledge, the fact $\predicate{Ill}(\constant{Mary},\constant{Aids})$ has to be deleted. But also  $\predicate{Ill}(\constant{Pete},\constant{Aids})$ can be deduced credulously, because it is satisfied by answer set $S_1$. In order to avoid this, we have three options:
 delete  $\predicate{Treat}(\constant{Pete},\constant{Medi1})$, delete the non-literal rule in \kb{} or insert $\predicate{Treat}(\constant{Pete},\constant{Medi2})$. The same solutions are found
 for $\kb$, $\policy'$ and \prior: they block the credulous deduction of $\neg\predicate{AbleToWork}(\constant{Pete})$. The same applies to $\kb'$ and \policy.
\end{example}

In the following sections we obtain a minimal solution $K^\mathit{pub}$ for a given input $K$, $\prior$ and $\policy$ by transforming the input into a problem of \emph{extended abduction} and solving it with an appropriate update program.

\section{Extended Abduction}\label{sec:extendedabduction}
Traditionally, given a knowledge base $K$ and an observation formula $O$, \emph{abduction} finds a ``(positive) explanation'' $E$ -- a set of hypothesis formulas -- such that every answer set of the knowledge base and the explanation together satisfy the observation; that is, $K\cup E\models O$.
Going beyond that \cite{inoue1995abductive,sakama2003abductive} use \emph{extended} abduction with the notions of ``negative observations'', ``negative explanations'' $F$ and ``anti-explanations''. 
An abduction problem in general can be restricted by specifying a designated set $\mathcal{A}$ of \emph{abducibles}. This set poses syntactical restrictions on the explanation sets $E$ and $F$. In particular, positive explanations are characterized by $E\subseteq \mathcal{A}\setminus K$ and negative explanations by $F\subseteq K\cap\mathcal{A}$.
If  $\mathcal{A}$ contains a formula with variables, it is meant as a shorthand for all ground instantiations of the formula.
In this sense, an EDP $\kb$ accompanied by an EDP  $\mathcal{A}$ is called an \emph{abductive program} written as $\langle \kb,\mathcal{A}\rangle$. The aim of extended abduction is then to find (anti-)explanations as follows (where in this article only \emph{skeptical} (anti-)explanations are needed):
\begin{itemize}
 \item given a \emph{positive} observation $O$, find a pair $(E,F)$ where $E$ is a positive explanation and $F$ is a negative explanation such that 
\begin{enumerate}
 \item \textbf{[skeptical explanation]} $O$ is satisfied in \emph{every} answer set of $(\kb\setminus F)\cup E$; that is, $(\kb\setminus F)\cup E\models O$
 \item \textbf{[consistency]} $(\kb\setminus F)\cup E$ is consistent
 \item \textbf{[abducibility]} $E\subseteq \mathcal{A}\setminus \kb$ and $F\subseteq \mathcal{A}\cap \kb$
\end{enumerate}
\item given a \emph{negative} observation $O$, find a pair $(E,F)$ where $E$ is a positive anti-explanation and $F$ is a negative anti-explanation such that 
\begin{enumerate}
 \item \textbf{[skeptical anti-explanation]} there is \emph{no} answer set of $(\kb\setminus F)\cup E$ in which $O$ is satisfied
 \item \textbf{[consistency]} $(\kb\setminus F)\cup E$ is consistent
 \item \textbf{[abducibility]} $E\subseteq \mathcal{A}\setminus \kb$ and $F\subseteq \mathcal{A}\cap \kb$
\end{enumerate}
\end{itemize}

 Among (anti-)explanations, \textbf{minimal}
(anti-)explanations characterize a subset-minimal alteration of the program $\kb$: an (anti-)explanation $(E, F)$ of an
observation $O$ is called minimal if for any (anti-)explanation $(E', F')$ of $O$, $E'\subseteq E$
and $F' \subseteq F$ imply $E' = E$ and $F' = F$.

For an abductive program $\langle \kb,\mathcal{A}\rangle$ both $\kb$ and $\mathcal{A}$ are semantically identified with their
ground instantiations with respect to the Herbrand universe, so that set operations over them are defined on the ground
instances. Thus, when $(E, F)$ contain formulas with variables, $(\kb \setminus F) \cup E$ means deleting
every instance of formulas in $F$, and inserting any instance of formulas in $E$ from/into $\kb$. When $E$
contains formulas with variables, the set inclusion $E'\subseteq E$ is defined for any set $E'$ of instances of formulas in $E$.
Generally, given sets $S$ and $T$ of literals/rules containing variables, any set operation
$\circ$ is defined as $S \circ T = \mathit{inst}(S)\circ\mathit{inst}(T)$ where $\mathit{inst}(S)$ is the ground instantiation of
$S$. For example, when $p(x) \in T$, for any constant $a$ occurring in $T$, it holds that $\{p(a)\} \subseteq T$, $\{p(a)\}\setminus T = \emptyset$, and $T \setminus\{p(a)\} = (T \setminus\{p(x)\}) \cup \{ p(y) \mid
y \not= a\}$, etc. Moreover, any literal/rule in a set is identified with its variants modulo
variable renaming.

\subsection{Normal form}\label{sec:normalform}

Although extended abduction can handle the very general format of EDPs, some syntactic transformations are helpful.
Based on \cite{sakama2003abductive}
we will briefly describe how a semantically equivalent normal form of an abductive program $\langle \kb,\mathcal{A}\rangle$ is obtained -- where both the program $\kb$ and the set $\mathcal{A}$ of abducibles are EDPs. 
This makes an automatic handling of abductive programs easier; for example, abductive programs in normal form can be easily transformed into update programs as described in Section~\ref{sec:updateprograms}.
The main step is that rules in $\mathcal{A}$ can be mapped to atoms by a naming function $n$. 
Let $\mathcal{R}$ be the set of abducible rules:
\begin{displaymath}
 \mathcal{R}=\{\Sigma\leftarrow\Gamma\mid (\Sigma\leftarrow\Gamma) \in \mathcal{A}\mbox{ and }(\Sigma\leftarrow\Gamma)\mbox{ is not a literal}\}
\end{displaymath}
Then the \emph{normal form} $\langle \kb^n,\mathcal{A}^n\rangle$ is defined as follows where $n(R)$ maps each rule $R$ to a fresh atom with the same free variables as $R$:
\begin{eqnarray*}
 \kb^n&=(\kb\setminus\mathcal{R})&\cup\{\Sigma\leftarrow\Gamma,n(R)\mid R=(\Sigma\leftarrow\Gamma)\in\mathcal{R}\}\\
&&\cup\{n(R)\mid R\in \kb\cap\mathcal{R}\}\\
\mathcal{A}^n&=(\mathcal{A}\setminus\mathcal{R})&\cup\{n(R)\mid R\in\mathcal{R}\}
\end{eqnarray*}
We define that any abducible literal $L$ has the name $L$,
i.e., $n(L) = L$. 
It is shown in \cite{sakama2003abductive}, that for
any observation $O$ there is a 1-1 correspondence between (anti-)explanations with
respect to $\langle \kb,A\rangle$ and those with respect to $\langle \kb^n,A^n\rangle$. That is, for $n(E) =
\{ n(R) | R \in E \}$ and $n(F) =
\{ n(R) | R \in F \}$:
 an observation
$O$ has a (minimal) skeptical (anti-)explanation $(E, F)$ with respect
to $\langle \kb,A\rangle$ iff $O$ has a (minimal) skeptical (anti-)explanation $(n(E), n(F))$
with respect to $\langle \kb^n,A^n\rangle$. Hence, insertion (deletion) of a rule's name in the normal form corresponds to insertion (deletion) of the rule in the original program. 
In sum, with the normal form transformation, any abductive program with abducible rules is reduced
to an abductive program with only abducible literals. 

\begin{example}\label{ex:normalform}
 We transform the example knowledge base \kb{} into its normal form based on a set of abducibles that is identical to \kb: that is $\mathcal{A}=\kb$; a similar setting will be used in Section~\ref{sec:onlyremovals} to achieve deletion of formulas from \kb.
Hence we transform $\langle \kb,\mathcal{A}\rangle$ into its normal form $\langle \kb^n,\mathcal{A}^n\rangle$ as follows where we write $n(R)$ for the naming atom of the only rule in $\mathcal{A}$:
\begin{eqnarray*}
 \kb^n&=&\{\predicate{Ill}(\constant{Mary},\constant{Aids}),~~~~~~\predicate{Treat}(\constant{Pete},\constant{Medi1}),~~~~~n(R),\\
&&\predicate{Ill}(x,\constant{Aids});\predicate{Ill}(x,\constant{Flu})\leftarrow\predicate{Treat}(x,\constant{Medi1}),\naf \predicate{Treat}(x,\constant{Medi2}),n(R)\}\\
\mathcal{A}^n&=&\{\predicate{Ill}(\constant{Mary},\constant{Aids}),~~~~\predicate{Treat}(\constant{Pete},\constant{Medi1}),~~~~~n(R)~~\}
\end{eqnarray*}
\end{example}


\subsection{Update programs}\label{sec:updateprograms}

Minimal (anti-)explanations can be computed with \emph{update programs} (UPs) \cite{sakama2003abductive}.
The \emph{update-minimal} (U-minimal) answer sets of a UP describe which rules have to be deleted from the program, and which rules have to be inserted into the program, in order (un-)explain an observation.

For the given EDP \kb{} and a given set of abducibles $\mathcal{A}$, a set of \textbf{update rules} $\mathit{UR}$ is devised that describe how entries of \kb{} can be changed. This is done with the following three types of rules.
\begin{enumerate}
 \item \textbf{[Abducible rules]} The rules for abducible literals state that an abducible is either true in \kb{} or not. 
For each $L\in\mathcal{A}$, a new atom $\bar{L}$ is introduced that has the same variables as $L$. Then the set of abducible rules for each $L$ is defined as
\begin{displaymath}
\mathit{abd}(L):=\{L\leftarrow \naf \bar{L}~,~\bar{L}\leftarrow \naf {L}\}.
\end{displaymath}
\item \textbf{[Insertion rules]} Abducible literals that are not contained in $\kb$ might be inserted into \kb{} and hence might occur in the set $E$ of the explanation $(E,F)$. For each $L\in\mathcal{A}\setminus \kb$, a new atom $+L$ is introduced and the insertion rule is defined as
\begin{displaymath}
+L\leftarrow L. 
\end{displaymath}
\item \textbf{[Deletion rules]} Abducible literals that are contained in $\kb$ might be deleted from \kb{} and hence might occur in the set $F$ of the explanation $(E,F)$. For each $L\in\mathcal{A}\cap \kb$, a new atom $-L$ is introduced and the deletion rule is defined as
\begin{displaymath}
-L\leftarrow \naf L. 
\end{displaymath}
\end{enumerate}
The \textbf{update program} is then defined by replacing abducible literals in \kb{} with the update rules; that is, 
\begin{displaymath}
UP=(\kb\setminus \mathcal{A})\cup UR.
\end{displaymath}

\begin{example}\label{ex:updateprog}
Continuing Example~\ref{ex:normalform}, from $\langle \kb^n,\mathcal{A}^n\rangle$ we obtain 
\begin{eqnarray*}
 UP&=\{&\mathit{abd}(\predicate{Ill}(\constant{Mary},\constant{Aids})),~~~\mathit{abd}(\predicate{Treat}(\constant{Pete},\constant{Medi1})),~~~\mathit{abd}(n(R)),\\
&&-\predicate{Ill}(\constant{Mary},\constant{Aids})\leftarrow\naf\predicate{Ill}(\constant{Mary},\constant{Aids}),\\
&&-\predicate{Treat}(\constant{Pete},\constant{Medi1})\leftarrow\naf\predicate{Treat}(\constant{Pete},\constant{Medi1}),\\
&&-n(R)\leftarrow\naf\, n(R),\\
&&\predicate{Ill}(x,\constant{Aids});\predicate{Ill}(x,\constant{Flu})\leftarrow\predicate{Treat}(x,\constant{Medi1}),\naf \predicate{Treat}(x,\constant{Medi2}),n(R)\}
\end{eqnarray*}
\end{example}

The set of atoms $+L$ is the set $\mathcal{UA}^+$ of positive update atoms; the set of atoms $-L$ is the set $\mathcal{UA}^-$ of negative update atoms. The set of \textbf{update atoms} is $\mathcal{UA}=\mathcal{UA}^+\cup\mathcal{UA}^-$. From all answer sets of an update program $\mathit{UP}$ we can identify those that are \textbf{update minimal} (U-minimal): they contain less update atoms than others. Thus, $S$ is U-minimal iff there is no answer set $T$ such that $T\cap\mathcal{UA}\subset S\cap\mathcal{UA}$.

\subsection{Ground observations}\label{sec:groundobservations}
It is shown in \cite{inoue1995abductive} how in some situations the observation formulas $O$ can be mapped to new positive ground observations.
Non-ground atoms with variables can be mapped to a new ground observation. Several positive observations can be conjoined and mapped to a new ground observation.
A negative observation (for which an anti-explanation is sought) can be mapped as a NAF-literal to a new positive observation (for which then an explanation has to be found).
Moreover, several negative observations can be mapped as a conjunction of NAF-literals to one new positive observation such that its resulting explanation acts as an anti-explanation for all negative observations together.
 Hence, in extended abduction it is usually assumed that $O$ is a positive ground observation for which an explanation has to be found.
In case of finding a skeptical explanation, an inconsistency check has to be made on the resulting knowledge base.
Transformations to a ground observation and inconsistency check will be detailed in Section~\ref{sec:policytrans} and applied to confidentiality-preservation.

\section{Confidentiality-Preservation with UPs}\label{sec:confidentialityups}

We now show how to achieve confidentiality-preservation by extended abduction: we define the set of abducibles and describe how a confidentiality-preserving knowledge base can be obtained by computing U-minimal answer sets of the appropriate update program.
We additionally distinguish between the case that we allow only deletions of formulas -- that is, in the anti-explanation $(E,F)$ the set $E$ of positive anti-explanation formulas is empty -- and the case that we also allow insertions.

\subsection{Policy transformation for credulous users}\label{sec:policytrans}
Elements of the confidentiality policy will be treated as negative observations for which an anti-explanation has to be found. 
Accordingly, we will transform policy elements to a set of rules containing new positive observations as sketched in Section~\ref{sec:groundobservations}.
We will call these rules \textbf{policy transformation rules for credulous users} ($\mathit{PTR}^\mathit{cred}$).

More formally, assume \policy{} contains $k$ elements. 
For each conjunction $C_i\in\policy$ ($i=1\mathellipsis k$), we introduce a new negative ground observation $O_i^-$ and map $C_i$ to $O_i^-$. As each  $C_i$ is a conjunction of (NAF-)literals, the resulting formula is an EDP rule.
As a last policy transformation rule, we add one that maps all new negative ground observations $O_i^-$ (in their NAF version) to a positive observation $O^+$.	Hence,

\begin{eqnarray*}
\mathit{PTR}^\mathit{cred}&:=&\{O^-_i\leftarrow C_i\mid C_i\in\policy\}\cup\{O^+\leftarrow \naf\,O_1^-,\mathellipsis,\naf\,O_k^-\}.
\end{eqnarray*}

\begin{example}\label{example:ptr}
The set of policy transformation rules for $\policy'$ is
\begin{displaymath}
\mathit{PTR}^\mathit{cred}=\{O^-_1\leftarrow	\predicate{Ill}(x,\constant{Aids})~,~O^-_2\leftarrow	\neg\predicate{AbleToWork}(x)~,~O^+\leftarrow	\naf\, O^-_1,\naf\, O^-_2\}
\end{displaymath}
\end{example}

Lastly, we consider a \textbf{goal rule} $\mathit{GR}$ that enforces the single positive observation $O^+$: $\mathit{GR}=\{\leftarrow \naf\, O^+\}$.

We can also allow more expressive policy elements in disjunctive normal form (DNF: a disjunction of conjunctions of (NAF-)literals).
If we map a DNF formula to a new observation (that is, $O^-_\mathit{disj}\leftarrow C_1\lor\mathellipsis\lor C_l$) this is equivalent to mapping each conjunct to the observation (that is, $O^-_\mathit{disj}\leftarrow C_1,\mathellipsis, O^-_\mathit{disj}\leftarrow C_l$).
We also semantically justify this splitting into disjuncts by arguing that in order to protect confidentiality of a disjunctive formula we indeed have to protect each disjunct alone. However, if variables are shared among disjuncts, these variables have to be grounded according to the Herbrand universe of $K\cup\prior$ first; otherwise the shared semantics of these variables is lost.

\subsection{Deletions for credulous users}\label{sec:onlyremovals}
As a simplified setting, we first of all assume that only deletions are allowed to achieve confidentiality-preservation. This setting can informally be described as follows:
For a given knowledge base $\kb$, if we only allow deletions of rules from \kb,
we have to find a  \emph{skeptical negative explanation} $F$ that explains the new positive observation $O^+$ while respecting \prior{} as invariable a priori knowledge.
The set of abducibles is thus identical to \kb{} as we want to choose formulas from \kb{} for deletion: $\mathcal{A}=\kb$.
That is, in total we consider the abductive program $\langle \kb,\mathcal{A}\rangle$.  Then, we transform it into normal form $\langle \kb^n,\mathcal{A}^n\rangle$, and  compute its update program $UP$ as described in Section~\ref{sec:updateprograms}.
As for \prior{}, we add this set to the update program $UP$ in order to make sure that the resulting answer sets of the update program do not contradict $\prior$.
Finally, we add all the policy transformation rules $\mathit{PTR}^\mathit{cred}$ and the goal rule $\mathit{GR}$.
The goal rule is then meant as a constraint that filters out those answer sets of $UP\cup\prior\cup\mathit{PTR}^\mathit{cred}$ in which $O^+$ is \emph{true}. 
 We thus obtain a new program $P$ as 
\begin{displaymath}
 P=UP\cup\prior\cup\mathit{PTR}^\mathit{cred}\cup\mathit{GR}
\end{displaymath}
 and compute its U-minimal answer sets. If $S$ is one of these answer sets, the negative explanation $F$ is obtained from the negative update atoms contained in $S$:
	$F=\{L\mid -L\in S\}$.

To obtain a confidentiality-preserving knowledge base for a credulous user, we have to check for inconsistency with the negation of the positive observation $O^+$ (which makes $F$ a \emph{skeptical} explanation of $O^+$); and allow only answer sets of $P$ that are U-minimal among those respecting this inconsistency property. More precisely, we check whether 
\begin{equation}\label{eqn:inconsistency}
(K\setminus F)\cup\prior\cup\mathit{PTR}^\mathit{cred}\cup\{\leftarrow O^+\}	\mbox{ is inconsistent.}
\end{equation}

\begin{example}\label{ex:up}
 We combine the update program $UP$ of $\kb$ with  $\prior$ and the policy transformation rules and goal rule.
This leads to the following two U-minimal answer sets with only deletions which satisfy the inconsistency property (\ref{eqn:inconsistency}):
\begin{eqnarray*}
S'_1&=&\{-\predicate{Ill}(\constant{Mary},\constant{Aids}),-\predicate{Treat}(\constant{Pete},\constant{Medi1}),n(R),
\overline{\predicate{Ill}(\constant{Mary},\constant{Aids})},\overline{\predicate{Treat}(\constant{Pete},\constant{Medi1})},O^+\}\\
S'_2&=&\{-\predicate{Ill}(\constant{Mary},\constant{Aids}),\predicate{Treat}(\constant{Pete},\constant{Medi1}),-n(R),\overline{\predicate{Ill}(\constant{Mary},\constant{Aids})},\overline{n(R)},O^+\}.
\end{eqnarray*}
These answer sets correspond to the minimal solutions from Example~\ref{ex:confidentialcredulous} where $\predicate{Ill}(\constant{Mary},\constant{Aids})$ must be deleted together with either $\predicate{Treat}(\constant{Pete},\constant{Medi1})$ or the rule named $R$. 
\end{example}

\begin{theorem}[Correctness for deletions]\label{theo:delete}
A knowledge base $K^\mathit{pub}=K\setminus F$ preserves confidentiality and changes $K$ subset-minimally iff $F$ is obtained by an answer set of the program $P$ that is U-minimal among those satisfying the inconsistency property (\ref{eqn:inconsistency}). 
\end{theorem}
\begin{proof}\emph{(Sketch)}
First of all note that because we chose $K$ to be the set of abducibles $\mathcal{A}$, only negative update atoms from $\mathcal{UA}^-$ occur in $\mathit{UP}$ -- no insertions with update atoms from $\mathcal{UA}^+$ will be possible. Hence we automatically obtain an anti-explanation $(E,F)$ where $E$ is empty. 
As shown in \cite{sakama2003abductive}, there is a 1-1 correspondence of minimal explanations and U-minimal answer sets of update programs; and anti-explanations are identical to explanations of a new positive observation when applying the transformations as in $\mathit{PTR}^\mathit{cred}$. By properties of skeptical (anti-)explanations we have thus $K^\mathit{pub}\cup\prior\cup\mathit{PTR}^\mathit{cred}\models O^+$ but for every $O_i^-$ there is no answer set in which $O_i^-$ is satisfied. This holds iff for every policy element $C_i$ there is no answer set of $K^\mathit{pub}\cup\prior$ that satisfies any instantiation of $C_i$ (with respect to the Herbrand universe of $K^\mathit{pub}\cup\prior$); thus 
$\mathit{cred}(K^\mathit{pub}\cup\prior,C_i)=\emptyset$. Subset-minimal change carries over from U-minimality of answer sets. 
\end{proof}

\subsection{Deletions and literal insertions}
To obtain a confidentiality-preserving knowledge base, (incorrect) entries may also be inserted into the knowledge base. 
To allow for insertions of literals, a more complex set $\mathcal{A}$ of abducibles has to be chosen. We reinforce the point that the subset $\mathcal{A}\cap K$ of abducibles that are already contained in the knowledge base $K$ are those that may be deleted while the subset  $\mathcal{A}\setminus K$ of those abducibles that are not contained in $K$ may be inserted.

First of all, we assume that the policy transformation is applied as described in Section~\ref{sec:policytrans}. Then, starting from the new negative observations $O_i^-$ used in the policy transformation rules, we trace back all rules in $K\cup\prior\cup\mathit{PTR}^\mathit{cred}$ that influence these new observations and collect all literals in the bodies of these rules. In other words, we construct a dependency graph (as in \cite{zhao2010usingasp}) and collect the literals that the negative observations depend on.
More formally, let $P_0$ be the set of literals that the new observations $O^-_i$ directly depend on:
\begin{eqnarray*}
	P_0&=\{L\mid& L\in\mathit{body}(R) \mbox{ or }\naf L \in\mathit{body}(R)\\
	&&\mbox{ where }R\in\mathit{PTR}^\mathit{cred}\mbox{ and }O^-_i\in\mathit{head}(R)\}
\end{eqnarray*}
Next we iterate and collect all the literals that the $P_0$ literals depend on:
\begin{eqnarray*}
	P_{j+1}&=\{L\mid& L\in\mathit{body}(R)\mbox{ or }\naf L \in\mathit{body}(R)\\
&&	\mbox{ where }R\in K\cup\prior\cup\mathit{PTR}^\mathit{cred}\mbox{ and }\mathit{head}(R)\cap P_j\neq\emptyset\}
\end{eqnarray*}
 and combine all such literals in a set $\mathcal{P}=\bigcup^\infty_{j=0}P_j$.

As we also want to have the option to delete rules from $K$ (not only the literals in $\mathcal{P}$), we define the set of abducibles as the set $\mathcal{P}$ plus all those rules in $K$ whose head depends on literals in $\mathcal{P}$:
\begin{displaymath}
\mathcal{A}=\mathcal{P}\cup\{R\mid R\in K \mbox{ and }\mathit{head}(R)\cap \mathcal{P}\neq \emptyset\}	
\end{displaymath}

\begin{example}
For the example $K\cup\prior\cup\mathit{PTR}^\mathit{cred}$, the dependency graph is shown in Figure~\ref{fig:dependencygraph}.
We note that the new negative observation $O^-_1$ directly depends on the literal $\predicate{Ill}(x,\constant{Aids})$ and the new negative observation $O^-_2$ directly depends on the literal $\neg\predicate{AbleToWork}(x)$; this is the first set of literals
$	P_0=\{\predicate{Ill}(x,\constant{Aids}),\neg\predicate{AbleToWork}(x)\}$.
By tracing back the dependencies in the graph,  
$\mathcal{P}=\{\predicate{Ill}(x,\constant{Aids}),\neg\predicate{AbleToWork}(x),\predicate{Ill}(x,\constant{Flu}),$ $\predicate{Treat}(x,\constant{Medi1}),\predicate{Treat}(x,\constant{Medi2})\}$ is obtained.
Lastly, we also have to add the rule $R$ from $K$ to $\mathcal{A}$ because literals in its head are contained in $\mathcal{P}$.
\end{example}

\begin{figure}
\begin{center}
\includegraphics{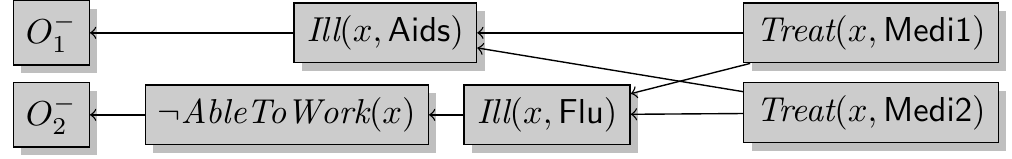}
\caption{Dependency graph for literals in $K\cup\prior\cup\mathit{PTR}$\label{fig:dependencygraph}}
\end{center}
\end{figure} 

We obtain the normal form and then the update program $\mathit{UP}$ for $K$ and the new set of abducibles $\mathcal{A}$.
The process of finding a skeptical explanation proceeds with finding an answer set of program $P$ as in Section~\ref{sec:onlyremovals} where additionally the positive explanation $E$ is obtained as
$	E=\{L\mid +L\in S\}$ and $S$ is U-minimal among those satisfying
\begin{equation}\label{eqn:inconsistencyinsert}
(K\setminus F)\cup E\cup\prior\cup\mathit{PTR}^\mathit{cred}\cup\{\leftarrow O^+\}	\mbox{ is inconsistent.}
\end{equation}

\begin{example}
For $\mathit{UP}$ from Example~\ref{ex:up} the new set of abducibles leads to additional insertion rules. Among others, the insertion rule for the new abducible $\predicate{Treat}(\constant{Pete},\constant{Medi2})$ is
$+\predicate{Treat}(\constant{Pete},\constant{Medi2})\leftarrow\predicate{Treat}(\constant{Pete},\constant{Medi2})$.
With this new rule included in $UP$, we also obtain the solution of Example~\ref{ex:confidentialcredulous} where the fact $\predicate{Treat}(\constant{Pete},\constant{Medi2})$ is inserted into $K$ (together with deletion of $\predicate{Ill}(\constant{Mary},\constant{Aids})$).
\end{example}

\begin{theorem}[Correctness for deletions \& literal insertions]
A knowledge base $K^\mathit{pub}=(K\setminus F)\cup E$ preserves confidentiality and changes $K$ subset-minimally iff $(E,F)$ is obtained by an answer set of program $P$ that is U-minimal among those satisfying inconsistency property (\ref{eqn:inconsistencyinsert}). 
\end{theorem}
\begin{proof}\emph{(Sketch)}
In $\mathit{UP}$, positive update atoms from $\mathcal{UA}^+$ occur for literals on which the negative observations depend.
For subset-minimal change, only these literals are relevant for insertions; inserting other literals will lead to non-minimal change. 
In analogy to Theorem~\ref{theo:delete}, by the properties of minimal skeptical (anti-)explanations that correspond to U-minimal answer sets of an update program, we obtain a confidentiality-preserving  $K^\mathit{pub}$ with minimal change.
\end{proof}

\section{Discussion and Conclusion}\label{sec:conclusion}
This article showed that when publishing a logic program, confidentiality-preservation can be ensured by extended abduction; more precisely, we showed that under the credulous query response it reduces to finding skeptical anti-explanations with update programs.
This is an application of data modification, because a user can be mislead by the published knowledge base to believe incorrect information; we hence apply dishonesties \cite{sakama2011dishonestreasoning} as a security mechanism. This is in contrast to \cite{zhao2010usingasp} whose aim is to avoid incorrect deductions while enforcing access control on a knowledge base. Another difference to \cite{zhao2010usingasp} is that they do not allow disjunctions in rule heads; hence, to the best of our knowledge this article is the first one to handle a confidentiality problem for EDPs.
In \cite{bonatti1995foundations} the authors study databases that may provide users with 
incorrect answers
to preserve security in a multi-user environment.  Different from our 
approach, they consider a database as a set of formulas of propositional logic and 
formulate the problem using modal logic.
In analogy to \cite{sakama2003abductive}, a complexity analysis for our approach can be achieved by reduction of extended abduction to normal abduction.
Work in progress covers data publishing for skeptical users; future work might handle insertion of non-literal rules.
\bibliographystyle{plain}
\bibliography{bib}
\end{document}